\UseRawInputEncoding
\documentclass[12pt]{article}

\usepackage{dsfont}
\usepackage{amsfonts,amsmath,amsthm}
\usepackage{algorithm}
\usepackage[algo2e]{algorithm2e}
\usepackage{color}
\usepackage{graphicx}
\usepackage{stmaryrd}        % provides \llbracket and \rrbracket
\usepackage{algorithm}
\usepackage{algorithmic}

\usepackage[paper=a4paper,dvips,top=2cm,left=2cm,right=2cm,
    foot=1cm,bottom=4cm]{geometry}

\begin{document}

\title{A regularization-patching dual quaternion optimization method for solving the hand-eye calibration problem}
\author{
    Zhongming Chen\footnote{%
    Department of Mathematics, School of Science, Hangzhou Dianzi University, Hangzhou 310018 China,
    ({\tt zmchen@hdu.edu.cn}).  This author's work was partially supported by Zhejiang Provincial Natural Science Foundation of China (No. LY22A010012).  }
    \and  \
    Chen Ling\footnote{%
    Department of Mathematics, School of Science, Hangzhou Dianzi University, Hangzhou 310018 China,
    ({\tt macling@hdu.edu.cn}).  This author's work was partially supported by National Natural Science
Foundation of China (No. 11971138) and City University of Hong Kong (Project 9610034).}
    \and  \
    Liqun Qi\footnote{%
    Department of Applied Mathematics, The Hong Kong Polytechnic University, Hung Hom,
    Kowloon, Hong Kong;
    Department of Mathematics, School of Science, Hangzhou Dianzi University, Hangzhou 310018 China; Center for Intelligent Multidimensional Data Analysis, Science Park, Shatin, Hong Kong, ({\tt maqilq@polyu.edu.hk}).  This author's work was supported by Hong Kong Innovation and Technology Commission (InnoHK Project CIMDA).}
    \and  \
    Hong Yan\footnote{%
    Department of Electrical Engineering, City University of Hong Kong, Kowloon, Hong Kong; Center for Intelligent Multidimensional Data Analysis, Science Park, Shatin, Hong Kong,
    ({\tt h.yan@cityu.edu.hk}).
    This author's work was supported by Hong Kong Innovation and Technology Commission (InnoHK Project CIMDA) and City University of Hong Kong (Project 9610034).}
}
\date{\today}
\maketitle

\begin{abstract}
The hand-eye calibration problem is an important application problem in robot research. Based on the 2-norm of dual quaternion vectors, we propose a new dual quaternion optimization method for the hand-eye calibration problem.  The dual quaternion optimization problem is decomposed to two quaternion optimization subproblems. The first quaternion optimization subproblem governs the rotation of the robot hand.  It can be solved efficiently by the eigenvalue decomposition or singular value decomposition. If the optimal value of the first quaternion optimization subproblem is zero, then the system is rotationwise noiseless, i.e., there exists a ``perfect'' robot hand motion which meets all the testing poses rotationwise exactly. In this case, we apply the regularization technique for solving the second subproblem to minimize the distance of the translation.   Otherwise we apply the patching technique to solve the second quaternion optimization subproblem. Then solving the second quaternion optimization subproblem turns out to be solving a quadratically constrained quadratic program.  In this way, we give a complete description for the solution set of hand-eye calibration problems. This is new in the hand-eye calibration literature. The numerical results are also presented to show the efficiency of the proposed method.

%The hand-eye calibration problem is an important application problem in robot research. Based on the 2-norm of dual quaternion vectors, we propose a new dual quaternion optimization model for the hand-eye calibration problem. To solve the problem, we first need to check whether the standard part of dual quaternion vectors are zeros or not. As a result, the dual quaternion optimization problem is divided into two quaternion optimization problems. \textcolor{red}{The first problem can be solved efficiently by the eigenvalue decomposition or singular value decomposition. Depending on the optimal value of the first problem, a regularized quaternion optimization problem or a patched quaternion optimazation problem is chosen. Both of them turn out to be quadratically constrained quadratic programs.} We give a complete description for the solution set of hand-eye calibration problems. The numerical results are also presented to show the efficiency of proposed methods.

\medskip

  \medskip


  %\medskip
  %\textbf{AMS subject classifications.}
\end{abstract}

\renewcommand{\Re}{\mathds{R}}
\newcommand{\rank}{\mathrm{rank}}
\renewcommand{\span}{\mathrm{span}}
\newcommand{\X}{\mathcal{X}}
\newcommand{\A}{\mathcal{A}}
\newcommand{\I}{\mathcal{I}}
\newcommand{\B}{\mathcal{B}}
\newcommand{\C}{\mathcal{C}}
\newcommand{\OO}{\mathcal{O}}
\newcommand{\e}{\mathbf{e}}
\newcommand{\0}{\mathbf{0}}
\newcommand{\dd}{\mathbf{d}}
\newcommand{\ii}{\mathbf{i}}
\newcommand{\jj}{\mathbf{j}}
\newcommand{\kk}{\mathbf{k}}
\newcommand{\va}{\mathbf{a}}
\newcommand{\vb}{\mathbf{b}}
\newcommand{\vc}{\mathbf{c}}
\newcommand{\vg}{\mathbf{g}}
\newcommand{\vr}{\mathbf{r}}
\newcommand{\vt}{\rm{vec}}
\newcommand{\vx}{\mathbf{x}}
\newcommand{\vy}{\mathbf{y}}
\newcommand{\vu}{\mathbf{u}}
\newcommand{\vv}{\mathbf{v}}
\newcommand{\y}{\mathbf{y}}
\newcommand{\vz}{\mathbf{z}}
\newcommand{\T}{\top}

\newtheorem{Thm}{Theorem}[section]
\newtheorem{Def}[Thm]{Definition}
\newtheorem{Ass}[Thm]{Assumption}
\newtheorem{Lem}[Thm]{Lemma}
\newtheorem{Prop}[Thm]{Proposition}
\newtheorem{Cor}[Thm]{Corollary}

\newpage
\section{Introduction}
The hand-eye calibration problem is an important part of robot calibration, which has wide applications in aerospace, medical, automotive and industrial fields \cite{JLL21, EFI21}. The problem is to determine the homogeneous matrix between the robot gripper and a camera mounted rigidly on the gripper or between a robot base and the world coordinate system. In 1989, Shiu and Ahmad \cite{SA89} and Tsai and Lenz \cite{TL89} used one motion (two poses) to formulate the hand-eye calibration problem as solving a matrix equation
\begin{equation} \label{hand_eye_eq1}
    A X = X B ,
\end{equation}
where $X$ is the unknown homogeneous transformation matrix from the gripper (hand) to the camera (eye), $A$ is the measurable homogeneous transformation matrix of the robot hand from its first to second position, and $B$ is the measurable homogeneous transformation matrix of the attached camera and also, from its first to second position.
To allow the simultaneous estimation of both the transformations from the robot base frame to the world frame and from the robot hand frame to sensor frame, Zhuang, Roth and Sudhaker \cite{ZRS94} derived another homogeneous transformation equation
\begin{equation} \label{hand_eye_eq2}
    A X = Z B ,
\end{equation}
where $X$ is the transformation matrix from the gripper to the camera, $Z$ is the transformation matrix from the robot base to the world coordinate system, $A$ is the transformation matrix from the robot base to the gripper and $B$ is the transformation matrix from the world base to the camera. It is worth mentioning that there are other kinds of mathematical models for hand-eye calibration problem. In this paper, we only focus on the models (\ref{hand_eye_eq1}) and (\ref{hand_eye_eq2}).

% Solution methods
Over the years, many different methods and solutions are developed for the hand-eye calibration problem. Based on how the rotation and translation parameters are estimated, these approaches are broadly divided into two categories: separable solutions and simultaneous solutions.
The separable solutions arise from solving the orientational component separately from the positional component. By using rotation matrix and translation vector to represent homogeneous transformation matrices, the hand-eye calibration equation is decomposed into rotation equation and position equation. The rotation parameters are first estimated. After that, the translation vectors could be estimated by solving a linear system. The different techniques that focus on the parametrization of rotation matrices include angle-axis \cite{SA89, TL89, W92}, Lie algebra \cite{PM94}, quaternions \cite{CK88, CK91, HD95}, Kronecker product \cite{LM08, S13} and so on. The main drawback in these methods is that rotation estimation errors propagate to position estimation errors.

On the other hand, the simultaneous solutions arise from simultaneously solving the orientational component and the positional component. The rotation and translation parameters are solved either analytically or by means of numerical optimization. For analytical approaches, many techniques were proposed including quaternions \cite{LC95}, screw motion \cite{C91}, Kronecker product \cite{AHE99}, dual tensor \cite{CB16}, dual Lie algebra \cite{CC19} and so on. The approaches based on numerical optimization include  Levenberg-Marquardt algorithm \cite{ZS93, RDL97}, gradient/Newton method \cite{GKP03}, linear matrix inequality \cite{HHP14}, alternative linear programming \cite{Z19}, and pseudo-inverse \cite{ZZY17}.
For more details about solution methods for hand-eye calibration problem, one can refer to \cite{EFI21, SEH12} and references therein.

%quaternion and dual quaternion methods
Among the solution methods for hand-eye calibration problem, the technique of dual quaternions was used to represent rigid motions by Daniilidis and  Bayro-Corrochano \cite{DB96}. Based on the dual-quaternion parameterization, a simultaneous solution for the hand-eye problem was proposed by using the singular value decomposition \cite{DB96, D99}. After that, many solution methods based on dual quaternions were proposed \cite{SVN03, LWW10, MB10, US16, LLDL18}. It has been shown that the dual quaternion representation gives a stable way to estimate the solution.

The existing methods for the hand-eye calibration problem used to produce solutions in general cases£¬ i.e., the rotation axes are not parallel.   There lacks a complete description of the solution set of the hand-eye calibration problem.   %The existing optimization methods for this problem are more or less established on some approximation bases.

In this paper, we propose a new dual quaternion optimization method for the hand-eye calibration problem based on the 2-norm of dual quaternion vectors, aiming to give a complete description of the solution set of the hand-eye calibration problem.

The theoretical base of dual quaternion optimization was established in \cite {QLY22}, where a total order for dual numbers, the magnitude of a dual quaternion and the norm for dual quaternion vectors were proposed. Then, a two-stage solution scheme for equality constrained dual quaternion optimization problems was proposed in \cite{Q22}, with the hand-eye calibration problem and the simultaneous localization and mapping problem as application examples.  It was shown in \cite{Q22} that an equality constrained dual quaternion optimization problem could be solved by solving two quaternion optimization subproblems.

In the solution scheme of \cite{Q22}, the optimization solution set of the first quaternion optimization subproblem is designed as a constraint of the second quaternion optimization subproblem.   This poses a challenge for implementing such a two-stage solution scheme in practice.   In this paper, we propose a regularization-patching method to solve such a dual quaternion optimization problem arising from the hand-eye calibration problem.   To apply the two-stage scheme of \cite{Q22} to the hand-eye calibration problem, we may solve the first quaternion optimization subproblem efficiently by the eigenvalue decomposition or singular value decomposition. If the optimal value of the first subproblem is equal to zero, a regularization function is used to solve the second quaternion optimization subproblem. Otherwise, the solution of the second subproblem is determined by solving a patched quaternion optimization problem.
In fact, the optimal value of the first subproblem is equal to zero if and only if there exists a ``perfect'' robot hand motion which meets all the testing poses exactly.  In this case, we say that the hand-eye calibration system is {\bf rotationwise noiseless}.  The flow chart of proposed method is presented in Figure \ref{fig0}. In this way, we give a complete description for the solution set of the hand-eye calibration problem. This is new in the hand-eye calibration literature and should be useful in applications.

\begin{figure}[!t]
\centering
\includegraphics[width=0.95\textwidth]{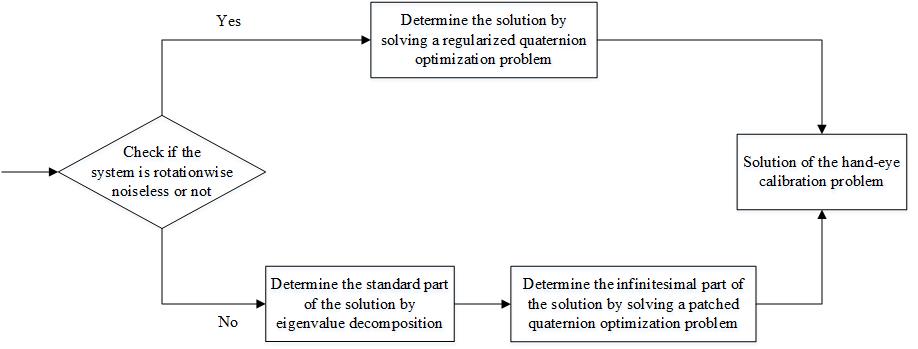}
\caption{The flow chart of proposed method.}
\label{fig0}
\end{figure}

In the next section, we present some preliminary knowledge on dual numbers, quaternions and dual quaternions. Based on dual quaternion optimization, the reformulations and analysis for hand-eye calibration equations $AX=XB$ and $AX=ZB$ are given in Sections 3 and 4, respectively. In Section 5, we present the numerical results to show the efficiency of proposed methods. Some final remarks are made in Section 6.

Throughout the paper, the sets of real numbers, dual numbers, quaternion numbers and  dual quaternion numbers are denoted by $\mathbb R$, $\mathbb D$, $\mathbb Q$ and $\mathbb{DQ}$, respectively.
The sets of $n$-dimensional real vectors, quaternion vectors and dual quaternion vectors are denoted by ${\mathbb R}^n$, ${\mathbb Q}^n$ and ${\mathbb{DQ}}^n$, respectively. Scalars, vectors and matrices are denoted by lowercase letters, bold lowercase letters and capital letters, respectively.

\section{Preliminaries}

\subsection{Dual numbers}

A dual number $q \in \mathbb{D}$ can be written as $q = q_{st} + q_\I\epsilon$, where $q_{st}, q_\I \in \mathbb{R}$  and $\epsilon$ is the infinitesimal unit satisfying $\epsilon^2 = 0$.   We call $q_{st}$ the standard part of $q$, and $q_\I$ the infinitesimal part of $q$.  Dual numbers can be added component-wise, and multiplied by the formula
$$
(p_{st} + p_\I\epsilon) (q_{st} + q_\I\epsilon) = p_{st} q_{st} +(p_{st} q_\I + p_\I q_{st})\epsilon .
$$
The dual numbers form a commutative algebra of dimension two over the reals.
%If $q_{st} \not = 0$, we say that $q$ is appreciable, otherwise, we say that $q$ is infinitesimal.
The absolute value of $q = q_{st} + q_{\I} \epsilon \in \mathbb{D}$ is defined as
\begin{equation*}
|q| = \left\{\begin{aligned} |q_{st}| + \frac{ q_{st} q_{\I} }{|q_{st}|}\epsilon, & \quad {\rm if}\ q_{st} \not = 0, \\
|q_{\I}|\epsilon, & \quad {\rm otherwise.} \end{aligned}\right.
\end{equation*}

A total order ``$\le$" for dual numbers was introduced in \cite{QLY22}.  Given two dual numbers $p, q \in \mathbb D$, $p = p_{st} + p_\I\epsilon$, $q = q_{st} + q_\I\epsilon$, where $p_{st}, p_\I, q_{st}, q_\I \in \mathbb{R}$, we say that $p \le q$, if either $p_{st} < q_{st}$, or $p_{st} = q_{st}$ and $p_\I \le q_\I$.  In particular, we say that $p$ is positive, nonnegative, nonpositive or negative, if $p > 0$, $p \ge 0$, $p \le 0$ or $p < 0$, respectively.

%\textcolor{red}{Suppose that $q = q_{st} + q_\I \epsilon \in \mathbb D$ and $q \geq 0$, the square root is defined as
%\begin{equation}
%\sqrt{q} = \left\{\begin{aligned} \sqrt{q_{st}} + \frac{q_\I}{2\sqrt{q_{st}}}\epsilon, & \quad {\rm if}\ q_{st} > 0, \\
%\sqrt{q_{\I}}\epsilon, & \quad {\rm otherwise.} \end{aligned}\right.
%\end{equation}
%Conventionally, we have $\sqrt{0} = 0$.}

\subsection{Quaternion numbers}

A quaternion number $q \in \mathbb{Q}$ has the form
$ q = q_0  + q_1 \ii + q_2 \jj + q_3 \kk , $
where $q_0, q_1, q_2, q_3 \in \mathbb{R}$ and $\ii, \jj, \kk$ are three imaginary units of quaternions satisfying
$$
\ii^2 = \jj^2 = \kk^2 = \ii \jj\kk=-1, \quad  \ii \jj = -\jj \ii = \kk,  \quad \jj\kk=-\kk\jj=\ii, \quad \kk\ii=-\ii\kk=\jj .
$$
The conjugate of $q$ is the quaternion $q^*=q_0  - q_1 \ii - q_2 \jj - q_3 \kk$. The scalar part of $q$ is $\text{Sc}(q) = \frac{1}{2}(q+q^*)=q_0$. % and the vector part of $q$ is $\text{Ve}(q)=\frac{1}{2} (q-q^*) = q_1 \ii + q_2 \jj + q_3 \kk$.
Clearly, $\text{Sc}(q^*) = \text{Sc}(q)$ and $(p q)^* = q^* p^*$ for any $p, q \in \mathbb{Q}$. The multiplication of quaternions is associative and distributive over vector addition, but is not commutative. The magnitude of $q$ is $$ |q| = \sqrt{q q^*} = \sqrt{q^* q}=\sqrt{q_0^2 + q_1^2 +q_2^2 + q_3^2}.$$

The quaternion $q \in \mathbb{Q}$ is called a unit quaternion if $|q|=1$. It is well-known \cite{WZ14} that the unit quaternion
$$ q = \cos\left(\frac{\theta}{2}\right) + \sin\left(\frac{\theta}{2}\right) n_1 \ii + \sin\left(\frac{\theta}{2}\right) n_2 \jj + \sin\left(\frac{\theta}{2}\right) n_3 \kk ,$$
can be used to described the rotation around a unit axis ${\bf n} = (n_1, n_2, n_3)^\top \in \mathbb{R}^3$ with an angle of $-\pi \leq \theta \leq \pi $.
On the other hand, given a unit quaternion $q = q_0  + q_1 \ii + q_2 \jj + q_3 \kk \in \mathbb{Q}$, the rotation matrix $R$ can be obtained by
\begin{equation}\label{eq_q2rot}
    R= \left(\begin{array}{ccc}
      q_0^2 + q_1^2-q_2^2-q_3^2 & 2(q_1 q_2 -q_0 q_3)       & 2(q_1 q_3 + q_0 q_2) \\
      2(q_1 q_2 +q_0 q_3)       & q_0^2 - q_1^2+q_2^2-q_3^2 & 2(q_2 q_3 - q_0 q_1) \\
      2(q_1 q_3 -q_0 q_2)       & 2(q_2 q_3 + q_0 q_1)      &  q_0^2 - q_1^2 - q_2^2 + q_3^2
    \end{array}  \right).
\end{equation}

For any $a = a_0 + a_1 \ii + a_2 \jj + a_3 \kk  \in \mathbb{Q}$, denote
$ \overrightarrow{a}= (a_0, a_1, a_2, a_3)^\top $
and
$$
M(a) = \left( \begin{array}{rrrr}
a_0 & -a_1  & -a_2  & -a_3  \\
a_1 &  a_0  & -a_3  &  a_2  \\
a_2 &  a_3  & a_0   & -a_1  \\
a_3 & -a_2  & a_1   & a_0  \end{array}
\right), \quad
W(a) = \left( \begin{array}{rrrr}
a_0 & -a_1  & -a_2  & -a_3  \\
a_1 &  a_0  &  a_3  & -a_2  \\
a_2 & -a_3  &  a_0  &  a_1  \\
a_3 &  a_2  & -a_1  &  a_0  \end{array}
\right).
$$
Clearly, $|a| = \| \overrightarrow{a} \|_2$. By direct calculations, we have the following propositions.
\begin{Prop}\label{p2-1}
For any $a = a_0 + a_1 \ii + a_2 \jj + a_3 \kk  \in \mathbb{Q}$ and $b = b_0 + b_1 \ii + b_2 \jj + b_3 \kk  \in \mathbb{Q}$, the following statements hold:
\begin{itemize}
    \item[(i)] ${\rm Sc}(r _1 a + r_2 b) = r_1 {\rm Sc}(a) + r_2 {\rm Sc}(b)$ for any $r_1, r_2 \in \mathbb{R}$.
    \item[(ii)] ${\rm Sc}(a^* b) = {\rm Sc}(a b^*) ={\rm Sc}(b^* a) = {\rm Sc}(b a^*) = \overrightarrow{a}^\top \overrightarrow{b}$.
    \item[(iii)] $M(a^*) = M(a)^\top$, $W(a^*) = W(a)^\top$.
    \item[(iv)] $ \overrightarrow{ab} = M(a) \overrightarrow{b} = W(b) \overrightarrow{a} $.
    \item[(v)] $M(a)^\top M(a) = W(a)^\top W(a) = \| \overrightarrow{a} \|_2^2 I_{4\times 4}$, where $I_{4\times 4}$ is the identity matrix of size $4 \times 4$.
\end{itemize}
\end{Prop}

\begin{Prop}\label{p2-2}
If $a$ and $b$ are two quaternion numbers satisfying $\text{Sc}(a^* b) = 0$, then for any $q \in \mathbb{Q}$, we have
$ \text{Sc} (q^* a^* b q) = \text{Sc} (q^* b^*  a q)  = 0.$
\end{Prop}
\begin{proof}
Since $\text{Sc}(a^* b) = 0$, we have $a^* b + b^*a = 0$. According to Proposition \ref{p2-1}, one can obtain that $ \text{Sc} (q^* a^* b q) = \text{Sc} (q^* b^*  a q)  = \text{Sc} \left( q^* (\frac{a^* b + b^* a}{2}) q \right) = 0$ for any $q \in \mathbb{Q}.$   \end{proof}

%quaternion vector and its norm
Next we introduce the $2$-norm for quaternion vectors, which can be found in \cite{QLY22}. Denote $\vx = (x_1, x_2, \cdots, x_n)^\top \in {\mathbb {Q}}^n$ for quaternion vectors.  The $2$-norm of $\vx \in {\mathbb {Q}}^n$ is defined as
\begin{equation*} %\label{2-norm-quaternion}
\|\vx \|_2 = \sqrt{\sum_{i=1}^n |x_i|^2 } = \sqrt{\sum_{i=1}^n \|\overrightarrow{x_i} \|_2^2 } .
\end{equation*}
The conjugate transpose of $\vx$ is defined as $\vx^* =  (x_1^*, x_2^*, \cdots, x_n^*)$. More details about quaternions and quaternion vectors could be found in \cite{WLZZ18}.
%Clearly, for any $\vx, \vy \in \mathbb{Q}^n$, we have $$ {\rm Sc}(\vx^* \vy) = \sum_{i=1}^n {\rm Sc} (x_i^* y_i) = \sum_{i=1}^n  \overrightarrow{x_i}^\top \overrightarrow{y_i}   \leq \| \vx\|_2 \| \vy \|_2. $$

%According to Proposition 6.3 of \cite{QLY22}, it holds that
%$$
%\| \vx \|_2 = \| \vx_{st} \|_2 + \frac{{\vx}_{st}^* {\vx_{\I}} + {\vx}_{\I}^*{\vx}_{st}}{2 \| {\vx}_{st}(x) \|_2} \epsilon
%$$
%for any $\vx \in \mathbb{DQ}^n$ with $\vx_{st} \neq {\bf 0}$.

\subsection{Dual quaternion numbers}

A dual quaternion number $q \in \mathbb{DQ}$ has the form $ q = q_{st} + q_\I \epsilon$, where $q_{st}, q_\I \in \mathbb{Q}$. %If $q_{st} \neq 0$, the dual quaternion $q$ is called appreciable.
The conjugate of $q = q_{st} + q_\I \epsilon$ is $q^* = q_{st}^* + q_\I^* \epsilon$.
%According to Proposition \ref{p2-1}, we have
%$$q^* q = q^*_{st} q_{st} + (q_{st}^* q_{\I} + q_{\I}^* q_{st}) \epsilon = |q_{st}|^2 + 2\text{Sc}( q_{st}^* q_{\I} ) \epsilon = |q_{st}|^2 + 2\text{Sc}(  q_{\I} q_{st}^*) \epsilon = q q^*.$$
The magnitude of a dual quaternion number $q = q_{st} + q_\I \epsilon$ is defined as
\begin{equation*}
|q| = \left\{\begin{aligned} |q_{st}| + \frac{{\rm Sc}( q_{st}^* q_{\I})}{|q_{st}|}\epsilon, & \quad {\rm if}\ q_{st} \not = 0, \\
|q_{\I}|\epsilon, & \quad {\rm otherwise.} \end{aligned}\right.
\end{equation*}
The dual quaternion number $q \in \mathbb{DQ}$ is called a unit dual quaternion if $|q|=1$.
Note that $q = q_{st} + q_{\I} \epsilon \in \mathbb{DQ}$ is a unit dual quaternion if and only if
$ q_{st}^* q_{st}= 1$ and $ q_{st}^* q_{\I} + q_{\I}^* q_{st} =  0 $.
According to Proposition \ref{p2-2}, we have the following result.
\begin{Cor} \label{c2-3}
If $q = q_{st} + q_{\I} \epsilon \in \mathbb{DQ}$ is a unit dual quaternion, then $\text{Sc}(q_{st}^* q_{\I}) = \text{Sc}(q_{\I}^* q_{st})  = 0$ and for any $a \in \mathbb{Q}$, we have
$ \text{Sc}( a^* q_{st}^* q_{\I} a ) = \text{Sc}( a^* q_{\I}^* q_{st} a )  = 0. $
\end{Cor}

%Relationship between unit dual quaternion and homogeneous tranformation matrix.
It has been shown that the 3D motion of a rigid body can be represented by a unit dual quaternion \cite{D99}. Consider a rigid motion in $SE(3)$ represented by a $4 \times 4$ homogeneous transformation matrix
\begin{equation}\label{eq_HTM}
T = \left( \begin{array}{cc} R & {\bf t} \\ {\bf 0}^\top & 1  \end{array}\right) ,
\end{equation}
where $R \in \mathbb{R}^{3 \times 3}$ is the rotation matrix about an axis through the origin and ${\bf t} \in \mathbb{R}^3$ is the translation vector. Let {$q_{st} \in \mathbb{Q}$} be the unit quaternion encoding the rotation matrix $R$ and let $t \in \mathbb{Q}$ be the quaternion satisfying $\overrightarrow{t} = \left( \begin{array}{c}
     0  \\
     {\bf t}
\end{array} \right)$. Then the transformation matrix $T$ is represented by the dual quaternion
$
{q = q_{st} +  q_\I \epsilon} ,
$
where {$q_\I = \frac{1}{2} t q_{st}$}. It is not difficult to check that $q \in \mathbb{DQ}$ is a unit dual quaternion since
$$ {\rm Sc}(q_{st}^* q_\I) = \frac{1}{2} {\rm Sc}(q_{st}^* t q_{st}) =0  .$$
On the other hand, given a unit dual quaternion $q = q_{st} + q_\I \epsilon \in \mathbb{DQ}$, the corresponding homogeneous transformation matrix $T$ can be obtained by (\ref{eq_HTM}), where the rotation matrix $R \in \mathbb{R}^{3 \times 3}$ can be derived from the unit quaternion $q_{st}$ according to (\ref{eq_q2rot}) and the translation vector ${\bf t} \in  \mathbb{R}^{3 }$ can be derived from
\begin{equation}\label{eq_q2tv}
    \left( \begin{array}{c}
     0  \\
     {\bf t}
\end{array} \right) = \overrightarrow{2 q_\I q_{st}^*} .
\end{equation}
It follows that $\|{\bf t}\|_2^2 = 4 q_\I q_{st}^* q_{st} q_\I^* =  4|q_\I|^2$. In other words, for a unit dual quaternion, the magnitude of its infinitesimal part is half of the length of the corresponding translation vector.

Denote $\vx = (x_1, x_2, \cdots, x_n)^\top \in {\mathbb {DQ}}^n$ for dual quaternion vectors.
We may also write
$$\vx = \vx_{st} + \vx_\I\epsilon,$$
where $\vx_{st}, \vx_\I \in {\mathbb Q}^n$.
%The vector $\vx \in {\mathbb {DQ}}^n$ is called appreciable if at least one of its components is appreciable, i.e., $\vx_{st} \neq \0$.
The $2$-norm of $\vx \in {\mathbb {DQ}}^n$ is defined as
\begin{equation} \label{2-norm}
\|\vx \|_2 = \left\{\begin{aligned}\sqrt{\sum_{i=1}^n |x_i|^2}, & \quad {\rm if}\ \vx_{st} \not = \0, \\
\|\vx_\I\|_2\epsilon, & \quad {\rm otherwise.} \end{aligned}\right.
\end{equation}
Denote by $\vx^* :=  (x_1^*, x_2^*, \cdots, x_n^*)$ the conjugate transpose of $\vx \in \mathbb{DQ}^n$. According to Proposition 6.3 of \cite{QLY22}, it holds that
\begin{equation}\label{eq_dqnorm}
\| \vx \|_2 = \| \vx_{st} \|_2 + \frac{ \text{Sc}({\vx}_{st}^* {\vx_{\I}})}{ \| {\vx}_{st} \|_2} \epsilon ,
\end{equation}
for any $\vx \in \mathbb{DQ}^n$ with $\vx_{st} \neq {\bf 0}$.

%\newpage

\section{Hand-Eye Calibration Equation $AX=XB$}
%\subsection{$AX=XB$}
The hand-eye calibration problem is to find the matrix $X$ such that
\begin{equation} \label{eq3-1}
 A^{(i)} X =X B^{(i)}  ,
\end{equation}
for $i=1,2,\ldots, n$, where $X$ is transformation matrix from the gripper (hand) to the camera (eye), $A^{(i)}$ is the transformation matrix between the grippers of two different poses and $B^{(i)}$ the transformation matrix between the cameras of two different poses.
%Reformulation by dual quaternion numbers. Notations:
%\begin{itemize}
%    \item ${a}^{(i)}={a}^{(i)}_{st} + a^{(i)}_{\I} \epsilon, \  {b}^{(i)} = b^{(i)}_{st} + b^{(i)}_{\I} \epsilon, \ x=x_{st} + x_{\I} \epsilon \in \mathbb{DQ}$
%    \item $f_i({x}) = {a}^{(i)} {x} - {x} {b}^{(i)} = \left(a^{(i)}_{st} x_{st}-x_{st}b_{st}^{(i)} \right) + \left( a_{st}^{(i)}x_{\I} +a_{\I}^{(i)}x_{st} - x_{st}b_{\I}^{(i)} - x_{\I}b_{st}^{(i)}   \right) \epsilon$
%    \item ${\bf f}({x}) = \left( f_1(x), f_2(x), \ldots, f_n(x) \right)^\top = {\bf f}_{st}({x}) + {\bf f}_{\I}({x})\epsilon$
%\end{itemize}
The transformation matrices $X$, $A^{(i)}$ and $B^{(i)}$ are encoded with the unit dual quaternions
$$ x=x_{st} + x_{\I} \epsilon, \quad {a}^{(i)}={a}^{(i)}_{st} + a^{(i)}_{\I} \epsilon, \quad  {b}^{(i)} = b^{(i)}_{st} + b^{(i)}_{\I} \epsilon,  $$
for $i=1,2,\ldots,n$.
Let ${\bf a} = \left(a^{(1)}, a^{(2)}, \ldots, a^{(n)} \right)^\top \in \mathbb{DQ}^n$ and ${\bf b} = \left(b^{(1)}, b^{(2)}, \ldots, b^{(n)} \right)^\top \in \mathbb{DQ}^n$.
The hand-eye calibration problem (\ref{eq3-1}) can be reformulated as the dual quaternion optimization problem
\begin{equation} \label{eq3-2}
\begin{array}{cl}
      \min       & \| {\bf a} x- x {\bf b} \|_2     \\
      \text{s.t.}& |x| = 1 .
\end{array}
\end{equation}

Denote ${\bf f}(x)={\bf a} x- x {\bf b}  \in \mathbb{DQ}^n$. According to (\ref{2-norm}) and (\ref{eq_dqnorm}), we have $$
\| {\bf f}(x) \|_2 = \left\{
\begin{array}{ll}
\| {\bf f}_{st}(x) \|_2 + \frac{ \text{Sc}\left( {\bf f}_{st}^*(x) {\bf f}_{\I}(x) \right) }{ \| {\bf f}_{st}(x) \|_2} \epsilon,   & \text{if } {\bf f}_{st}(x) \neq {\bf 0}, \\
\| {\bf f}_{\I}(x) \|_2 \epsilon,        & \text{otherwise}.
\end{array}
\right.
$$
Problem (\ref{eq3-2}) can be divided to two different cases, which need to be handled very differently. One case is that the standard part of the optimal value of (\ref{eq3-2}) is zero.  Another case is that the standard part of the optimal value of (\ref{eq3-2}) is positive.     Physically, the standard part of the optimal value of (\ref{eq3-2}) is zero if and only if there exists a ``perfect'' robot hand motion $x$, which meets all the $n$ testing poses rotationwise exactly. In this case, we say that system is {\bf rotationwise noiseless}.    The following proposition provides a way to check if the system is rotationwise noiseless or not.
   %we need to check whether there exists a unit dual quaternion $x$ such that ${\bf f}_{st}(x) =0$, which could be done by the following proposition.

\begin{Prop}\label{pro1}
If $\hat{x}$ is an optimal solution of (\ref{eq3-2}), the standard part $\hat{x}_{st}$ is an optimal solution of the quaternion optimization problem
\begin{equation} \label{eq3-3}
\begin{array}{cl}
      \min       & \| {\bf a}_{st} x_{st}- x_{st} {\bf b}_{st} \|_2^2     \\
      \text{s.t.}& |x_{st}| = 1 .
\end{array}
\end{equation}
Hence, the system is rotationwise noiseless if and only if the optimal value of (\ref{eq3-3}) is equal to  zero.
\end{Prop}
\begin{proof}
According to the definition of total order for dual numbers, the result could be easily proved since ${\bf f}_{st}(x) = {\bf a}_{st} x_{st}- x_{st} {\bf b}_{st} $ .
\end{proof}

%\red{Physically, if problem (\ref{eq3-2}) is calm, then there exists a perfect rotation which meets all the testing poses exactly.   In general, the situation is not so ideal, i.e., there exits noises. However, near perfect situations always exist, i.e., there are cases when noises are very small.  Thus, we should consider all the cases.}

Denote the optimal set of (\ref{eq3-3}) by $X_{st}$. If the optimal value of (\ref{eq3-3}) is equal to  zero,
we consider the regularized quaternion optimization problem
\begin{equation} \label{eq3-4}
\begin{array}{cl}
      \min       & \| {\bf f}_{\I} (x) \|_2^2  + \gamma (x_{st}^* x_{st} + x_{\I}^* x_\I )   \\
      \text{s.t.}& x_{st} \in X_{st}, \quad x_{st}^*x_{\I} + x_{\I}^*x_{st} = 0 .
\end{array}
\end{equation}
where $\gamma$ is the parameter that balances the loss function and the regularization term. In fact, $x_{st} \in X_{st}$ implies $x_{st}^* x_{st} =1$, and $x_{\I}^* x_\I$ is proportional to the norm square of translation vector. By adding the regularization term, we try to find the best solution with minimal distance of translation. {\bf This explains the role of regularization.}

If the optimal value of (\ref{eq3-3}) is not equal to zero, we consider the quaternion optimization problem
\begin{equation} \label{eq3-5}
\begin{array}{cl}
      \min       &  \text{Sc}\left( {\bf f}_{st}^*(x) {\bf f}_{\I}(x) \right) \\ %+ {\bf f}_{\I}^*(x){\bf f}_{st}(x)    \\
      \text{s.t.}& x_{st} \in X_{st}, \quad x_{st}^*x_{\I} + x_{\I}^*x_{st} = 0 .
\end{array}
\end{equation}

By using the matrix representation for quaternion numbers, problems (\ref{eq3-3}), (\ref{eq3-4}) and (\ref{eq3-5}) could be solved efficiently. For $i=1,2,\ldots, n$, we have
$$
\overrightarrow{ a^{(i)}_{st} x_{st}-x_{st}b_{st}^{(i)} }= \left[ M\left(a^{(i)}_{st}\right)-W\left(b_{st}^{(i)}\right) \right] \overrightarrow{ x_{st}}
$$
and
$$
\left| a^{(i)}_{st} x_{st}-x_{st}b_{st}^{(i)} \right|^2 = \overrightarrow{ x_{st}}^\top  \left[ M\left(a^{(i)}_{st}\right)-W\left(b_{st}^{(i)}\right) \right]^\top \left[ M\left(a^{(i)}_{st}\right)-W\left(b_{st}^{(i)}\right) \right] \overrightarrow{ x_{st}}  .
$$
Denote
\begin{equation} \label{eq_L11}
L_{11} = \sum_{i=1}^n \left[ M\left(a^{(i)}_{st}\right)-W\left(b_{st}^{(i)}\right) \right]^\top \left[ M\left(a^{(i)}_{st}\right)-W\left(b_{st}^{(i)}\right) \right].
\end{equation}
It follows that
$$
\| {\bf a}_{st} x_{st}- x_{st} {\bf b}_{st} \|_2^2 = \sum_{i=1}^n \left| a^{(i)}_{st} x_{st}-x_{st}b_{st}^{(i)} \right|^2 = \overrightarrow{x_{st}}^\top L_{11} \overrightarrow{x_{st}}.
$$
Denote the minimal eigenvalue of matrix $L_{11}$ by $\lambda_0$. As a result, problem (\ref{eq3-3}) is equivalent to finding the unit eigenvectors corrresponding to $\lambda_0$.

Similarly, for $i=1,2,\ldots, n$, we have
$$
\overrightarrow{ a_{st}^{(i)}x_{\I} +a_{\I}^{(i)}x_{st} - x_{st}b_{\I}^{(i)} - x_{\I}b_{st}^{(i)} }=  \left[ M\left(a^{(i)}_{st}\right)-W \left(b_{st}^{(i)} \right) \right] \overrightarrow{ x_{\I}} + \left[ M\left(a^{(i)}_{\I}\right)-W\left(b_{\I}^{(i)}\right) \right] \overrightarrow{ x_{st}} .
$$
Denote
\begin{equation} \label{eq_L22}
L_{22} = \sum_{i=1}^n \left[ M \left(a^{(i)}_{\I}\right)-W \left(b_{\I}^{(i)}\right)  \right]^\top \left[ M\left(a^{(i)}_{\I}\right)-W\left(b_{\I}^{(i)}\right)  \right]
\end{equation}
and
\begin{equation} \label{eq_L12}
L_{12} = \sum_{i=1}^n \left[ M \left(a^{(i)}_{st}\right)-W\left(b_{st}^{(i)}\right)  \right]^\top \left[ M\left(a^{(i)}_{\I}\right)-W\left(b_{\I}^{(i)}\right)  \right].
\end{equation}
It follows that
$$
\| {\bf f}_{\I} (x) \|_2^2 = \sum_{i=1}^n \left| a_{st}^{(i)}x_{\I} +a_{\I}^{(i)}x_{st} - x_{st}b_{\I}^{(i)} - x_{\I}b_{st}^{(i)} \right|^2 = \overrightarrow{x_{\I}}^\top L_{11} \overrightarrow{x_{\I}} + 2 \overrightarrow{x_{\I}}^\top L_{12} \overrightarrow{x_{st}} + \overrightarrow{x_{st}}^\top L_{22} \overrightarrow{x_{st}}.
$$
As a result, problem (\ref{eq3-4}) is equivalent to the optimization problem
\begin{equation} \label{eq3-6}
\begin{array}{cl}
      \min       & \overrightarrow{x_{\I}}^\top (L_{11} + \gamma I) \overrightarrow{x_{\I}} + 2 \overrightarrow{x_{\I}}^\top L_{12} \overrightarrow{x_{st}} + \overrightarrow{x_{st}}^\top (L_{22} + \gamma I) \overrightarrow{x_{st}}     \\
      \text{s.t.}& \overrightarrow{x_{st}} \in \overrightarrow{X_{st}}, \quad \overrightarrow{ x_{st}}^\top \overrightarrow{ x_{\I}} = 0 ,
\end{array}
\end{equation}
where $\overrightarrow{X_{st}}$ is the set of all the unit eigenvectors corrresponding to the minimal eigenvalue of matrix $L_{11}$. Once the set $\overrightarrow{X_{st}}$ is determinated, problem (\ref{eq3-6}) turns out to be a quadratically constrained quadratic program (QCQP).

To be specific, suppose that the dimension of the eigenspace of the minimal eigenvalue of $L_{11}$ is $k$. Let $Q \in \mathbb{R}^{4 \times k}$ be the matrix whose columns form an orthonormal basis of the eigenspace, i.e., $Q^\top Q = I_{k \times k}$. It is not difficult to see that $\overrightarrow{X_{st}} = \{ Q \y: \y^\top \y=1, \y \in \mathbb{R}^k \}$. Problem (\ref{eq3-6}) can be rewritten as
\begin{equation} \label{eq3-7}
\begin{array}{cl}
      \min       & \overrightarrow{x_{\I}}^\top (L_{11} + \gamma I) \overrightarrow{x_{\I}} + 2 \overrightarrow{x_{\I}}^\top L_{12} Q \y + \y^\top Q^\top (L_{22} + \gamma I) Q \y     \\
      \text{s.t.}& {\y}^\top \y =1, \quad \y^\top Q^\top \overrightarrow{ x_{\I}} = 0 .
\end{array}
\end{equation}
In particular, if the dimension of the eigenspace is one, i.e., $k=1$, the solution set $\overrightarrow{X_{st}} = \{ {\bf q}, -{\bf q}\} $, where $\bf q \in \mathbb{R}^4$ is the normalized basis of the eigenspace. In this case, problem (\ref{eq3-7}) could be solved efficiently by representing $\overrightarrow{ x_{\I}}$ in the orthogonal complement space of $\bf q$.

% To be specific, suppose that the eigenspace of the minimal eigenvalue of $L_{11}$ has dimension $k$. Let $Q_{1} \in \mathbb{R}^{4 \times k}$ be the matrix whose columns form an orthonormal basis of the eigenspace, i.e., $Q_{1}^\top Q_{1} = I_{k \times k}$, and $Q_{2} \in \mathbb{R}^{4 \times (4-k)}$ be the orthogonal complement of $Q_{1}$, i.e., $Q_{2}^\top Q_{1} = 0_{(4-k) \times k} $ and $Q_{2}^\top Q_{2} = I_{(4-k) \times (4-k)}$. It is not difficult to see that $\overrightarrow{X_{st}} = \{ Q_1 y: y^\top y=1, y\in \mathbb{R}^k \}$.

In the following, we reformulate problem (\ref{eq3-5}) as an optimization problem by using the matrix representation for quaternion numbers. According to Proposition \ref{p2-1}, we have
\begin{equation*}
\begin{aligned}
 & \text{Sc} \left( \left(a^{(i)}_{st} x_{st}-x_{st}b_{st}^{(i)} \right)^* \left( a_{st}^{(i)}x_{\I} +a_{\I}^{(i)}x_{st} - x_{st}b_{\I}^{(i)} - x_{\I}b_{st}^{(i)} \right)  \right) \\
 & =  \overrightarrow{ x_{st}}^\top \left[ M\left(a^{(i)}_{st}\right)-W\left(b_{st}^{(i)}\right) \right]^\top \left[ M\left(a^{(i)}_{st}\right)-W\left(b_{st}^{(i)}\right) \right] \overrightarrow{ x_{\I}} \\
 & \hspace{4cm} +  \overrightarrow{ x_{st}}^\top \left[ M\left(a^{(i)}_{st}\right)-W\left(b_{st}^{(i)}\right) \right]^\top \left[ M\left(a^{(i)}_{\I}\right)-W\left(b_{\I}^{(i)}\right) \right] \overrightarrow{ x_{st}}
\end{aligned}
\end{equation*}
for $i=1,2,\ldots, n$. It follows that
\begin{equation*}
\begin{aligned}
 \text{Sc} \left( {\bf f}_{st}^*(x) {\bf f}_{\I}(x) \right) & = \sum_{i=1}^n \text{Sc} \left(   \left(a^{(i)}_{st} x_{st}-x_{st}b_{st}^{(i)} \right)^* \left( a_{st}^{(i)}x_{\I} +a_{\I}^{(i)}x_{st} - x_{st}b_{\I}^{(i)} - x_{\I}b_{st}^{(i)} \right)  \right) \\
& = \overrightarrow{ x_{st}}^\top L_{11} \overrightarrow{ x_{\I}} + \overrightarrow{ x_{st}}^\top L_{12} \overrightarrow{ x_{st}} ,
\end{aligned}
\end{equation*}
where $L_{11}$ and $L_{12}$ are given by (\ref{eq_L11}) and (\ref{eq_L12}) respectively. Note that $\overrightarrow{X_{st}}$ is the set of all unit eigenvectors corresponding to the minimal eigenvalue $\lambda_0$ of $L_{11}$. Under the constraints of (\ref{eq3-5}), one can obtain that
$$
\overrightarrow{ x_{st}}^\top L_{11} \overrightarrow{ x_{\I}} = \lambda_0 \overrightarrow{ x_{st}}^\top \overrightarrow{ x_{\I}} = 0,
$$
since $L_{11}$ is symmetric. It turns out that problem (\ref{eq3-5}) is equivalent to the optimization problem
\begin{equation} \label{eq3-8}
\begin{array}{cl}
      \min       & \overrightarrow{ x_{st}}^\top L_{12} \overrightarrow{ x_{st}}     \\
      \text{s.t.}& \overrightarrow{x_{st}} \in \overrightarrow{X_{st}}, \quad \overrightarrow{ x_{st}}^\top \overrightarrow{ x_{\I}} = 0 .
\end{array}
\end{equation}
Similarly, if $Q$ is the matrix whose columns form an orthonormal basis of the eigenspace of $\lambda_0$, the optimal $\overrightarrow{x_{st}}$ can be derived by computing the unit eigenvectors corresponding to the minimal eigenvalue of $\text{Sym}\left(Q^\top L_{12} Q\right) = \left( Q^\top L_{12} Q + Q^\top L_{12}^\top Q \right)/2$. %, while the optimal $\overrightarrow{x_{\I}}$ can be any vector which is orthogonal to the optimal $\overrightarrow{x_{st}}$ in the sense of 2-norm for dual quaternion vectors.
Since the objective function in (\ref{eq3-8}) does not contain $\overrightarrow{x_\I}$ , the optimal $\overrightarrow{x_\I}$ can be any vector which is orthogonal to the optimal $\overrightarrow{x_{st}}$. We may need to find a proper one via sewing a patch on the optimal set of $\overrightarrow{x_\I}$, while ensuring that $\| {\bf f}_{st}(x) \|_2$ is minimized. Considering the continuity of the norm, it is naturally necessary to further search for $x_\I$ under the constrains of $\overrightarrow{x_{st}}^\top \overrightarrow{x_{\I}} = 0$, such that $\| {\bf f}_\I (x)\|_2$ is reduced as much as possible, i.e.,
%In practice, once the optimal $\overrightarrow{x_{st}}$ is determined, we try to find the best one in the optimal set of $\overrightarrow{x_\I}$ such that $\| {\bf f}_{\I} (x) \|_2^2$ is minimized, i.e.,
\begin{equation} \label{eq3-8+}
\begin{array}{cl}
      \min_{\overrightarrow{x_\I}}       & \overrightarrow{x_{\I}}^\top L_{11}  \overrightarrow{x_{\I}} + 2 \overrightarrow{x_{\I}}^\top L_{12} \overrightarrow{x_{st}} + \overrightarrow{x_{st}}^\top L_{22}  \overrightarrow{x_{st}}     \\
      \text{s.t.}& \overrightarrow{x_{st}}^\top \overrightarrow{x_{\I}} = 0 .
\end{array}
\end{equation}
{\bf This explains the role of the patching.}

Note that in this way, we give a complete description for the solution set of the hand-eye calibration problem. This is new in the hand-eye calibration literature and should be useful in applications.

To conclude, the solution method for hand-eye calibration equation $AX=XB$ is summarized in Algorithm \ref{alg_1}.
\begin{algorithm}[!ht]
\caption{Dual quaternion optimization for $AX = XB$ }
\begin{algorithmic}[1] \label{alg_1}
\REQUIRE{Motions $\left(A^{(i)}, B^{(i)}\right)_{i=1}^n$, regularization parameter $\gamma$. }
\ENSURE{The hand-eye transformation matrix $X$.}
\STATE{Construct the matrix $L_{11}, L_{22}$ and $L_{12}$ according to (\ref{eq_L11}), (\ref{eq_L22}) and (\ref{eq_L12}), respectively. }
\STATE{Compute the minimal eigenvalue $\lambda_0$ of $L_{11}$, and deduce the orthonormal basis $Q$ for the eigenspace of $\lambda_0$.}
\IF{ $\lambda_0 = 0 $ }
\STATE{Compute $x_{st}$ and $x_\I$ by solving QCQP (\ref{eq3-7}).}
\ELSE
\STATE{Compute $x_{st}$ by finding the unit eigenvector corresponding to the minimal eigenvalue of $\text{Sym}\left(Q^\top L_{12} Q \right)$.}
\STATE{Compute $x_\I$ by solving (\ref{eq3-8+}) with the optimal $x_{st}$.}
\ENDIF
\STATE{Compute $X = \left( \begin{array}{cc} R & {\bf t} \\ {\bf 0}^\top & 1 \end{array} \right)$, where $R$ is computed from $x_{st}$ by (\ref{eq_q2rot}) and $\bf t$ is computed from $x_{st}$ and $x_\I$ by using (\ref{eq_q2tv}).}
\end{algorithmic}
\end{algorithm}

%\red{The regularization and the patching techniques are necessary, otherwise we may have an unbounded solution set, which are not stable.}

%% Physical meaning of the solution

\section{Hand-Eye Calibration Equation $AX=ZB$}
In 1994, Zhuang, Roth and Sudhaker \cite{ZRS94} generalized (\ref{hand_eye_eq1}) to $AX=ZB$, where $X$ is transformation matrix from the gripper to the camera, $Z$ is the transformation matrix from the robot base to the world coordinate system, $A$ is the transformation matrix from the robot base to the gripper and $B$ is the transformation matrix from the world base to the camera. Given $n$ measurements $\left(A^{(i)}, B^{(i)}\right)_{i=1}^n$, the problem is to find the best solution $X$ and $Z$ such that
\begin{equation} \label{eq3-9}
 A^{(i)} X =Z B^{(i)}  ,
\end{equation}
for $i=1,2,\ldots, n$. The transformation matrices $X$, $Z$, $A^{(i)}$ and $B^{(i)}$ are encoded with the unit dual quaternions
$$ x=x_{st} + x_{\I} \epsilon, \quad z=z_{st} + z_{\I} \epsilon, \quad {a}^{(i)}={a}^{(i)}_{st} + a^{(i)}_{\I} \epsilon, \quad  {b}^{(i)} = b^{(i)}_{st} + b^{(i)}_{\I} \epsilon,  $$
for $i=1,2,\ldots,n$.
Let ${\bf a} = \left(a^{(1)}, a^{(2)}, \ldots, a^{(n)} \right)^\top \in \mathbb{DQ}^n$ and ${\bf b} = \left(b^{(1)}, b^{(2)}, \ldots, b^{(n)} \right)^\top \in \mathbb{DQ}^n$.
%Reformulation by dual quaternion numbers. Notations:
%\begin{itemize}
%    \item ${a}^{(i)}={a}^{(i)}_{st} + a^{(i)}_{\I} \epsilon, \  {b}^{(i)} = b^{(i)}_{st} + %b^{(i)}_{\I} \epsilon, \ x=x_{st} + x_{\I} \epsilon \in \mathbb{DQ}$
%    \item $g_i({x,z}) = {a}^{(i)} {x} - {z} {b}^{(i)} = \left(a^{(i)}_{st} x_{st}-z_{st}b_{st}^{(i)} \right) + \left( a_{st}^{(i)}x_{\I} +a_{\I}^{(i)}x_{st} - z_{st}b_{\I}^{(i)} - z_{\I}b_{st}^{(i)}   \right) \epsilon$
%    \item ${\bf g}(x,z) = \left( g_1(x, z), g_2(x, z), \ldots, g_n(x, z) \right)^\top = {\bf g}_{st}(x,z) + {\bf g}_{\I}(x,z)\epsilon$
%\end{itemize}
The hand-eye calibration problem (\ref{eq3-9}) can be reformulated as the dual quaternion optimization problem
\begin{equation} \label{eq3-10}
\begin{array}{cl}
      \min       & \| {\bf a} x- z {\bf b} \|_2     \\
      \text{s.t.}& |x| = |z| = 1 .
\end{array}
\end{equation}

Similarly, we say that the system is rotationwise noiseless if and only if
the standard part of the optimal value of (\ref{eq3-10}) is zero.

Denote ${\bf g}(x, z)={\bf a} x- z {\bf b} \in \mathbb{DQ}^n$. To solve problem (\ref{eq3-10}), according to the definition of 2-norm for dual quaternion vectors, we first consider the quaternion optimization problem
\begin{equation} \label{eq3-11}
\begin{array}{cl}
      \min       & \|{\bf g}_{st}(x, z) \|_2^2 = \| {\bf a}_{st} x_{st}- z_{st} {\bf b}_{st} \|_2^2     \\
      \text{s.t.}& |x_{st}| = |z_{st}|  = 1 .
\end{array}
\end{equation}
Note that $a = a_{st} + a_{\I} \epsilon \in \mathbb{DQ}$ is a unit dual quaternion if and only if
$ a_{st}^* a_{st}= 1$  and $ \text{Sc}\left(a_{st}^* a_{\I} \right) = a_{st}^* a_{\I} + a_{\I}^* a_{st} =  0 $.
For $i=1,2,\ldots, n$, we have
\begin{equation*}
\begin{aligned}
\left| a^{(i)}_{st} x_{st}-z_{st}b_{st}^{(i)}  \right|^2 & = \left(a^{(i)}_{st} x_{st}-z_{st}b_{st}^{(i)} \right)^* \left(a^{(i)}_{st} x_{st}-z_{st}b_{st}^{(i)} \right)  \\
& = 2 - 2 \text{Sc} \left( x_{st}^* \left(a_{st}^{(i)}\right)^* z_{st}b_{st}^{(i)}  \right) \\
& = 2 - 2  \overrightarrow{x_{st}}^\top M\left(a_{st}^{(i)}\right)^\top W \left(b_{st}^{(i)}\right) \overrightarrow{z_{st}}
\end{aligned}
\end{equation*}
since $x$, $z$, $a^{(i)}$ and $b^{(i)}$ are unit dual quaternions. Denote
\begin{equation}\label{eq_K11}
K_{11} = \sum_{i=1}^n M\left(a_{st}^{(i)}\right)^\top W\left(b_{st}^{(i)}\right) .
\end{equation}
It follows that
$$
\| {\bf a}_{st} x_{st}- z_{st} {\bf b}_{st} \|_2^2 = \sum_{i=1}^n \left| a^{(i)}_{st} x_{st}-z_{st}b_{st}^{(i)}  \right|^2 = 2n -2 \overrightarrow{x_{st}}^\top K_{11} \overrightarrow{z_{st}} .
$$
Then problem (\ref{eq3-11}) is equivalent to the optimization problem
\begin{equation} \label{eq3-12}
\begin{array}{cl}
      \max       & \overrightarrow{x_{st}}^\top K_{11} \overrightarrow{z_{st}}     \\
      \text{s.t.}& \overrightarrow{x_{st}}^\top \overrightarrow{x_{st}} = \overrightarrow{z_{st}}^\top \overrightarrow{z_{st}} = 1 .
\end{array}
\end{equation}
Denote the maximal singular value of $K_{11}$ by $\sigma_1$, the set of optimal vector pairs of (\ref{eq3-12}) by $\overrightarrow{\Omega_{st}}$. As a result, problem (\ref{eq3-11}) is to find the unit singular vector pairs for $\sigma_1$, which can be solved efficiently by the singular value decomposition (SVD).

If the optimal value of (\ref{eq3-11}) is equal to zero, i.e., $\sigma_1 = n$, consider the regularized optimization problem
\begin{equation} \label{eq3-13}
\begin{array}{cl}
      \min       & \|{\bf g}_{\I}(x, z) \|_2^2 + \gamma \left( x_{st}^* x_{st} + x_{\I}^* x_\I  + z_{st}^* z_{st} + z_{\I}^* z_\I \right)   \\
      \text{s.t.}&  \left(\overrightarrow{x_{st}}, \overrightarrow{z_{st}} \right) \in \overrightarrow{\Omega_{st}}, \quad \overrightarrow{ x_{st}}^\top \overrightarrow{ x_{\I}} = 0, \quad \overrightarrow{ z_{st}}^\top \overrightarrow{ z_{\I}} = 0,
\end{array}
\end{equation}
where $\gamma$ is the regularization parameter and $$\|{\bf g}_{\I}(x, z) \|_2^2 = \sum_{i=1}^n \left\| M\left(a_{st}^{(i)}\right) \overrightarrow{x_{\I}} + M\left(a_{\I}^{(i)}\right) \overrightarrow{x_{st}} - W\left(b_{\I}^{(i)}\right) \overrightarrow{z_{st}}  - W\left(b_{st}^{(i)}\right)  \overrightarrow{z_{\I}} \right\|_2^2 .$$
Once the set $\overrightarrow{\Omega_{st}}$ is determined, problem (\ref{eq3-13}) could be also written as an QCQP. To be specific, suppose the singular value decomposition of matrix $K_{11}$ is $K_{11}=U \Sigma V^\top$, where $U, V \in \mathbb{R}^{4 \times 4}$ are orthogonal and $\Sigma \in \mathbb{R}^{4 \times 4}$ is diagonal. Let $Q_1 \in \mathbb{R}^{4 \times k}$ be the matrix whose columns are the columns of $U$ corresponding to $\sigma_1$, and let $Q_2 \in \mathbb{R}^{4 \times k}$ be the matrix whose columns are the columns of $V$ corresponding to $\sigma_1$. It is not difficult to see that $ \overrightarrow{\Omega_{st}} = \left\{ \left( Q_1 \y, Q_2 \y \right) : \y^\top \y=1, \y \in \mathbb{R}^k \right\}$. In fact, for any unit vectors $\y_1$ and $\y_2$, the value of objective function of (\ref{eq3-12}) at the point $\left(\overrightarrow{x_{st}}, \overrightarrow{z_{st}} \right) = \left( Q_1 \y_1, Q_2 \y_2 \right)$ is
$$\overrightarrow{x_{st}}^\top K_{11} \overrightarrow{z_{st}} = \y_1^\top Q_1^\top K_{11} Q_2 \y_2 = \sigma_1 \y_1^\top \y_2 \leq \sigma_1, $$ according to the Cauchy-Schwarz inequality. Without loss of generality, we assume $\sigma_1 > 0$. Then the equality holds if and only if $\y_1 = \y_2$. As a result,
problem (\ref{eq3-13}) can be rewritten as an QCQP:
\begin{equation} \label{eq3-14}
\begin{array}{cl}
      \min       & \displaystyle \sum_{i=1}^n \left\| M\left(a_{st}^{(i)}\right) \overrightarrow{x_{\I}} + M\left(a_{\I}^{(i)}\right) Q_1 \vy  - W\left(b_{\I}^{(i)}\right) Q_2 \vy  - W\left(b_{st}^{(i)}\right)  \overrightarrow{z_{\I}} \right\|_2^2  + \gamma \left(\overrightarrow{ x_{\I}}^\top \overrightarrow{ x_{\I}} + \overrightarrow{ z_{\I}}^\top \overrightarrow{ z_{\I}} + 2 \y^\top \y \right) \\
      \text{s.t.}& \vy^\top \vy = 1, \quad \vy^\top Q_1^\top \overrightarrow{ x_{\I}} = 0, \quad \vy^\top Q_2^\top \overrightarrow{ z_{\I}} = 0.
\end{array}
\end{equation}
In particular, when $k=1$, problem (\ref{eq3-14}) could be solved efficiently by representing $\overrightarrow{ x_{\I}}$ and $\overrightarrow{ z_{\I}} $ in the corresponding orthogonal complement space of $Q_1$ and $Q_2$, respectively.

On the other hand, if the optimal value of (\ref{eq3-11}) is not equal to zero, consider the optimization problem
\begin{equation} \label{eq3-15}
\begin{array}{cl}
      \min       & \text{Sc}\left({\bf g}_{st}^*(x, z) {\bf g}_{\I}(x, z) \right) \\ %+ {\bf g}_{\I}^*(x, z){\bf g}_{st}(x, z)    \\
      \text{s.t.}&  \left(\overrightarrow{x_{st}}, \overrightarrow{z_{st}} \right) \in \overrightarrow{\Omega_{st}}, \quad \overrightarrow{ x_{st}}^\top \overrightarrow{ x_{\I}} = 0, \quad \overrightarrow{ z_{st}}^\top \overrightarrow{ z_{\I}} = 0.
\end{array}
\end{equation}
According to Corollary \ref{c2-3}, we have
$$
\text{Sc} \left( \left(a^{(i)}_{st} x_{st}\right)^* a_{st}^{(i)}x_{\I} \right) =
\text{Sc} \left( \left(a^{(i)}_{st} x_{st}\right)^* a_{\I}^{(i)}x_{st} \right) =
\text{Sc} \left( \left(z_{st} b_{st}^{(i)}\right)^* z_{st}b_{\I}^{(i)} \right) =
\text{Sc} \left( \left(z_{st} b_{st}^{(i)}\right)^* z_{\I}b_{st}^{(i)} \right) = 0
$$
since $x$, $z$, $a^{(i)}$ and $b^{(i)}$ are unit quaternions for $i=1,2,\ldots, n$.
It follows that
\begin{equation*}
\begin{aligned}
 & \text{Sc} \left( \left(a^{(i)}_{st} x_{st}-z_{st}b_{st}^{(i)} \right)^* \left( a_{st}^{(i)}x_{\I} +a_{\I}^{(i)}x_{st} - z_{st}b_{\I}^{(i)} - z_{\I}b_{st}^{(i)} \right)  \right)  \\
 & = - \text{Sc} \left( \left(a^{(i)}_{st} x_{st}\right)^* z_{st}b_{\I}^{(i)} + \left(a^{(i)}_{st} x_{st}\right)^* z_{\I}b_{st}^{(i)} + \left(z_{st} b_{st}^{(i)}\right)^* a_{st}^{(i)}x_{\I} + \left(z_{st} b_{st}^{(i)}\right)^* a_{\I}^{(i)}x_{st} \right) \\
 & = -  \overrightarrow{x_{st}}^\top \left[ M \left(a_{st}^{(i)}\right)^\top W\left(b_{\I}^{(i)}\right) + M\left(a_{\I}^{(i)}\right)^\top W\left(b_{st}^{(i)}\right) \right] \overrightarrow{z_{st}}  \\
 & \hspace{4cm} - \overrightarrow{x_{st}}^\top M\left(a_{st}^{(i)}\right)^\top W\left(b_{st}^{(i)}\right) \overrightarrow{z_{\I}} - \overrightarrow{x_{\I}}^\top M\left(a_{st}^{(i)}\right)^\top W\left(b_{st}^{(i)}\right) \overrightarrow{z_{st}}   .
\end{aligned}
\end{equation*}
Denote
\begin{equation}\label{eq_K12}
K_{12} = \sum_{i=1}^n M\left(a_{st}^{(i)}\right)^\top W\left(b_{\I}^{(i)}\right)
\end{equation}
and
\begin{equation}\label{eq_K21}
K_{21} = \sum_{i=1}^n M\left(a_{\I}^{(i)}\right)^\top W\left(b_{st}^{(i)}\right) .
\end{equation}
By simple computation, one can obtain that
\begin{equation*}
\begin{aligned}
\text{Sc}\left( {\bf g}_{st}^*(x, z) {\bf g}_{\I}(x, z)  \right) & =  \sum_{i=1}^n \text{Sc}\left( \left(a^{(i)}_{st} x_{st}-z_{st}b_{st}^{(i)} \right)^* \left( a_{st}^{(i)}x_{\I} +a_{\I}^{(i)}x_{st} - z_{st}b_{\I}^{(i)} - z_{\I}b_{st}^{(i)} \right)  \right)  \\
& = - \left[ \overrightarrow{ x_{st}}^\top (K_{12} + K_{21} ) \overrightarrow{ z_{st}} + \overrightarrow{ x_{st}}^\top K_{11}  \overrightarrow{ z_{\I}} + \overrightarrow{ x_{\I}}^\top K_{11}  \overrightarrow{ z_{st}}  \right] ,
\end{aligned}
\end{equation*}
where $K_{11}$, $K_{12}$ and $K_{21}$ are given by (\ref{eq_K11}), (\ref{eq_K12}) and (\ref{eq_K21}), respectively. Under the constraints of problem (\ref{eq3-15}), $\overrightarrow{x_{st}}$ and $\overrightarrow{z_{st}}$ are left-singular and right-singular vectors corresponding to the maximal singular value $\sigma_1$ for $K_{11}$, which means
$$
K_{11} \overrightarrow{z_{st}} = \sigma_1 \overrightarrow{x_{st}} \quad \text{and} \quad K_{11}^\top \overrightarrow{x_{st}} = \sigma_1 \overrightarrow{z_{st}} .
$$
Then we have $\overrightarrow{ x_{st}}^\top K_{11}  \overrightarrow{ z_{\I}} = \sigma_1 \overrightarrow{ z_{st}}^\top   \overrightarrow{ z_{\I}} = 0$ and $\overrightarrow{ x_{\I}}^\top K_{11}  \overrightarrow{ z_{st}} = \sigma_1 \overrightarrow{ x_{\I}}^\top  \overrightarrow{ x_{st}} = 0$ under the constraints of problem (\ref{eq3-15}).
As a result, problem (\ref{eq3-15}) is equivalent to the optimization
\begin{equation} \label{eq3-16}
\begin{array}{cl}
      \max       & \overrightarrow{ x_{st}}^\top (K_{12} + K_{21} ) \overrightarrow{ z_{st}}   \\
      \text{s.t.}&  \left(\overrightarrow{x_{st}}, \overrightarrow{z_{st}} \right) \in \overrightarrow{\Omega_{st}}, \quad \overrightarrow{ x_{st}}^\top \overrightarrow{ x_{\I}} = 0, \quad \overrightarrow{ z_{st}}^\top \overrightarrow{ z_{\I}} = 0.
\end{array}
\end{equation}
Similarly, given the singular value decomposition $K_{11} = U \Sigma V^\top$, let $Q_1$ be the matrix whose columns are the columns of $U$ corresponding to $\sigma_1$, and let $Q_2$ be the matrix whose columns are the columns of $V$ corresponding to $\sigma_1$. The optimal $\overrightarrow{x_{st}}$ and $\overrightarrow{z_{st}}$ can be derived by computing the unit eigenvectors corresponding to the maximal eigenvalue of $\text{Sym}\left( Q_1^\top (K_{12}+K_{21}) Q_2 \right)$. Since the objective funcion in (\ref{eq3-16}) does not contain $ \overrightarrow{ x_{\I}}$ and $ \overrightarrow{ z_{\I}}$, the optimal $\overrightarrow{x_{\I}}$ can be any vector which is orthogonal to the optimal $\overrightarrow{x_{st}}$, and the optimal $\overrightarrow{z_{\I}}$ can be any vector which is orthogonal to the optimal $\overrightarrow{z_{st}}$. Considering the continuity of the norm, once the optimal $\overrightarrow{x_{st}}$ and $\overrightarrow{z_{st}}$ are determined, we try to find the best one in the optimal set of $\overrightarrow{x_\I}$ and $\overrightarrow{z_\I}$ such that the patching function $\| {\bf g}_{\I} (x, z) \|_2^2$ is minimized, i.e.,
\begin{equation} \label{eq3-16+}
\begin{array}{cl}
      \max_{\overrightarrow{x_{\I}}, \overrightarrow{z_{\I}}}   & \displaystyle \sum_{i=1}^n \left\| M\left(a_{st}^{(i)}\right) \overrightarrow{x_{\I}} + M\left(a_{\I}^{(i)}\right) \overrightarrow{x_{st}} - W\left(b_{\I}^{(i)}\right) \overrightarrow{z_{st}}  - W\left(b_{st}^{(i)}\right)  \overrightarrow{z_{\I}} \right\|_2^2  \\
      \text{s.t.}& \overrightarrow{ x_{st}}^\top \overrightarrow{ x_{\I}} = 0,  \quad \overrightarrow{ z_{st}}^\top \overrightarrow{ z_{\I}} = 0.
\end{array}
\end{equation}

To conclude, the solution method for hand-eye calibration equation $AX=ZB$ is summarized in Algorithm \ref{alg_2}.
\begin{algorithm}[!ht]
\caption{Dual quaternion optimization for $AX = ZB$ }
\begin{algorithmic}[1] \label{alg_2}
\REQUIRE{Measurements $\left(A^{(i)}, B^{(i)}\right)_{i=1}^n$, regularization parameter $\gamma$. }
\ENSURE{The hand-eye transformation matrix $X$ and robort-word transformation matrix $Z$.}
\STATE{Construct the matrix $K_{11}, K_{12}$ and $K_{21}$ according to (\ref{eq_K11}), (\ref{eq_K12}) and (\ref{eq_K21}), respectively. }
\STATE{Compute SVD for $K_{11}$, and deduce the maximal singular value $\sigma_1$ with corresponding column-orthogonal matrices $Q_1$ and $Q_2$.}
\IF{ $\sigma_1 = n $ }
\STATE{Compute $x_{st}, z_{st}, x_\I$ and $z_\I$ by solving QCQP (\ref{eq3-14}).}
\ELSE
\STATE{Compute $x_{st}$ and $z_{st}$ by finding the unit eigenvector corresponding to the maximal eigenvalue of $\text{Sym}\left( Q_1^\top (K_{12}+K_{21}) Q_2 \right)$.}
\STATE{Compute $x_\I$ and $z_\I$ by solving (\ref{eq3-16+}) with the optimal $x_{st}$ and $z_{st}$.}
\ENDIF
\STATE{Compute $X$ and $Z$ from the dual quaternions $x=x_{st} + x_\I \epsilon $ and $z = z_{st} + z_\I \epsilon $ respectively.}
\end{algorithmic}
\end{algorithm}

%%%%%%%%%%%%%%%%%%%%%%%%%%%%%%%%%%%

\section{Numerical Experiments}
In this section, we report a set of synthetic experiments to show the efficiency of proposed methods for hand-eye calibration problem. All the codes are written in Matlab R2017a. The numerical experiments were done on a
desktop with an Intel Core i5-2430M CPU dual-core processor running at 2.4GHz and 6GB of RAM.

In the implementation of our proposed methods, we use GloptiPoly \cite{HLL09} to construct SDP relaxations of QCQPs, and use MOSEK \cite{Mosek} as SDP solver. Further, GloptiPoly can also recover the solution to the original problem and certify its optimality. We set the regularization parameter $\gamma = 2 \times 10^{-6}$. %for Algorithm \ref{alg_1} and $\gamma = 3$ for Algorithm \ref{alg_2}.
For hand-eye calibration model $AX=XB$, we compare our method with the direct estimation proposed by Tsai et al. \cite{TL89} (denoted by ``{\bf Tsai89}"), the Kronecker method proposed by Andreff et al. \cite{AHE99} (denoted by ``{\bf Andreff99}"), the classic dual quaternion method proposed by Daniilidis \cite{D99} (denoted by ``{\bf Daniilidis99}"), the improved dual quaternion method proposed by Malti et al. \cite{MB10} (denoted by ``{\bf Malti10}"), and the dual quaternion method using polynomial optimization proposed by Heller et al. \cite{HHP14} (denoted by ``{\bf Heller14}").

For hand-eye calibration model $AX=ZB$, we compare our method with the quaternion method proposed by Zhuang et al. \cite{ZRS94} (denoted by ``{\bf Zhuang94}"), the quaternion method proposed by Dornaika et al. \cite{DH98} (denoted by ``{\bf Dornaika98}"), the classic dual quaternion method proposed by Li et al. \cite{LWW10} (denoted by ``{\bf Li10}"), the dual quaternion method using polynomial optimization proposed by Heller et al. \cite{HHP14} (denoted by ``{\bf Heller14}"), and the dual quaternion method proposed by Li et al. \cite{LLDL18} (denoted by ``{\bf Li18}").

Numerical experiments are carried out as follows. First, the original homogeneous transformation matrices $\hat{X}$ and $\hat{Z}$ in (\ref{hand_eye_eq2}) are given by
\begin{equation} \label{eq_realX}
    \hat{X} = \left( \begin{array}{rrrr}
        0.9995  & -0.0100  &  0.0297  &  9.190   \\
        0.0116  &  0.9986  & -0.0523  &  5.397   \\
       -0.0291  &  0.0526  &  0.9982  &  0       \\ %62.628  \\
             0  &       0  &       0  &  1.0000  \\
\end{array}
\right),
\end{equation}
and
\begin{equation} \label{eq_realZ}
\hat{Z} = \left( \begin{array}{rrrr}
    0.2790  & -0.0981 &  -0.9553 &   164.226   \\
   -0.5439  &  0.8037 &  -0.2414 &   301.638   \\
    0.7914  &  0.5869 &   0.1709 &   0         \\ %-962.841 \\
         0  &       0 &        0 &   1.0000  \end{array}
\right).
\end{equation}
Second, we generate $n$ transformation matrices $A^{(i)}$, $i=1,2,\ldots, n$. Then the transformation matrix $B^{(i)}$ is computed by $B^{(i)} = \hat{Z}^{-1} A^{(i)} \hat{X}$ for $i=1,2,\ldots, n$. We use different methods to solve the hand-eye calibration equation $AX=ZB$ with the given matrices $\left(A^{(i)}, B^{(i)}\right)_{i=1}^n$.
For hand-eye calibration equation $AX=XB$, we construct $\frac{n(n-1)}{2}$ pairs of matrices $\left( \left(A^{(i)}\right)^{-1} A^{(j)}, \left(B^{(i)}\right)^{-1} B^{(j)} \right)_{i<j}$, denoted by $\left(\tilde{A}^{(s)}, \tilde{B}^{(s)}\right)_{s=1}^{n(n-1)/2}$. Then different methods are used to solve the hand-eye calibration equation $AX=XB$ with the given matrices $\left(\tilde{A}^{(s)}, \tilde{B}^{(s)}\right)_{s=1}^{n(n-1)/2}$.
%The estimation error is computed by $ e_X = \| X- X^* \|_2.$
The estimation errors are computed by
$$
e_X = \| X- \hat{X} \|_2,  \quad  e_Z = \| Z- \hat{Z} \|_2.
$$

%In the following subsections, three different ways are adopted to generate the transformation matrices $A^{(i)}$, $i = 1,2,\ldots, n$. The numerical results are also reported.

\subsection{Measurements with non-parallel rotation axis}
Four measurements of $A$ with non-parallel rotation axis are given by
\begin{equation*}
    A^{(1)} = \left( \begin{array}{rrrr}
    0.1752  & -0.6574  &  0.7329  & -10.5536 \\
    0.6325  & -0.4954  & -0.5954  & -30.5304 \\
    0.7545  &  0.5679  &  0.3290  &  50.4851 \\
         0  &       0  &       0  &  1.0000
\end{array}
\right),
\end{equation*}
\begin{equation*}
    A^{(2)} = \left( \begin{array}{rrrr}
   -0.0745  &  0.9661  &  0.2471  & -20.4123 \\
    0.8573  & -0.0645  &  0.5108  & -50.8904 \\
    0.5094  &  0.2499  & -0.8234  &  80.8685 \\
         0  &       0  &       0  &  1.0000
\end{array}
\right),
\end{equation*}
\begin{equation*}
    A^{(3)} = \left( \begin{array}{rrrr}
   -0.1456  & -0.6867  &  0.7122  & -20.5519 \\
    0.8252  & -0.4814  & -0.2955  & -30.6491 \\
    0.5458  &  0.5447  &  0.6367  &  60.4312 \\
         0  &       0  &       0  &  1.0000
\end{array}
\right),
\end{equation*}
\begin{equation*}
    A^{(4)} = \left( \begin{array}{rrrr}
   -0.1434  & -0.5250  &  0.8389  & -10.5892 \\
    0.8158  & -0.5427  & -0.2001  & -50.6730 \\
    0.5603  &  0.6557  &  0.5061  &  80.4641 \\
         0  &       0  &       0  &  1.0000
\end{array}
\right).
\end{equation*}
As described above, we have four measurements $\left(A^{(i)}, B^{(i)}\right)$ for equation $AX=ZB$, and six motions $\left(\tilde{A}^{(s)}, \tilde{B}^{(s)}\right)$ for equation $AX=XB$. The numerical results for $AX=XB$ and $AX=ZB$ with non-parallel rotation axis are reported in Tables \ref{tb_1} and \ref{tb_2}, respectively. The proposed Algorithms \ref{alg_1} and \ref{alg_2} show the best behavior in terms of estimation error.
Note that the first three methods in Table \ref{tb_1} and the first three methods in Table \ref{tb_2} get the solution via solving linear equations, while the other methods need to call SDP solvers to get the solution. That explains why Algorithm \ref{alg_1} may need more computation time to get the solution when compared with the first three methods in Table \ref{tb_1}, and Algorithm \ref{alg_2} may need more computation time to get the solution when compared with the first three methods in Table \ref{tb_2}.

\begin{table}[!th]
\caption{Numerical results for $AX=XB$ with non-parallel rotation axis} \vspace{2mm}
\label{tb_1}
\centering
\begin{tabular}{c|cccccc} \hline
        & {\bf Tsai89}  & {\bf Andreff99} & {\bf Daniilidis99}  & {\bf Malti10}   & {\bf Heller14}  &  {\bf Alg. \ref{alg_1} }   \\ \hline
$e_X$   & 0.0030 & 0.0027 & 0.0014 & 0.0019 &  0.0014 &  0.0003       \\
Time(s) & 0.0419 & 0.0188 & 0.0747 & 3.8888 &  1.4171 &  1.1649       \\
\hline
\end{tabular}
\end{table}

% the average time of 100 runs:
%   0.0022    0.0013    0.0044    0.8310    0.8834    0.7549

\begin{table}[!th]
\caption{Numerical results for $AX=ZB$ with non-parallel rotation axis} \vspace{2mm}
\label{tb_2}
\centering
\begin{tabular}{c|cccccc} \hline
        & {\bf Zhuang94}   & {\bf Dornaika98}  & {\bf Li10}   &  {\bf Heller14} & {\bf Li18}  &  {\bf Alg. \ref{alg_2}}   \\ \hline
$e_X$   & 0.0010 & 0.0362 & 0.0005 &   0.0012 &  0.0029  &  0.0004   \\
$e_Z$   & 0.0135 & 0.0712 & 0.0155 &   0.0138 &  0.0142  &  0.0132   \\
Time(s) & 0.0200 & 0.0190 & 0.0761 &  40.3641 &  1.9992  &  1.0494   \\
\hline
\end{tabular}
\end{table}

\subsection{Measurements with parallel rotation axis}
In this subsection, we test our algorithms for the case that all the axes of measurements are parallel, which is often the situation for the han-eye calibration of SCARA robots \cite{US16}. In this case, it has been shown that the problem is not well-defined and there exists a 1D manifold of equivalent solutions with identical algebraic error \cite{C91, Z98}. To evaluate the quality of solutions, we try to find the solution such that the third component of its translation vector is equal to zero, and then compare it with the real solution $\hat{X}$ and $\hat{Z}$ given by (\ref{eq_realX}) and (\ref{eq_realZ}), respectively.

Four measurements of $A$ are generated with the same rotation axis, but with different angles. Without loss of generality, the normalized rotation axis is ${\bf n} = ( 0, 0, 1)^\top $.
For $A^{(1)}, A^{(2)}, A^{(3)}, A^{(4)}$, the rotation angles are $\theta_1 = \frac{\pi}{6}$, $\theta_2=\frac{\pi}{3}$, $\theta_3 = -\frac{\pi}{6}$ and $\theta_4 -\frac{\pi}{3}$, while their translation vectors are randomly generated given by
$${\bf t}_1 = (-10.9865, 12.3788, -27.2571)^\top, \quad {\bf t}_2 = (38.8986, 84.6736, -93.8814)^\top , $$
$${\bf t}_3 = (-75.7189, -53.6187, 28.5794)^\top, \quad {\bf t}_4 = (-52.8133, 93.3732, -70.1666)^\top,$$
respectively. The numerical results for $AX=XB$ and $AX=ZB$ with parallel rotation axis are reported in Tables \ref{tb_3} and \ref{tb_4}, respectively.

\begin{table}[!th]
\caption{Numerical results for $AX=XB$ with parallel rotation axis} \vspace{2mm}
\label{tb_3}
\centering
\begin{tabular}{c|cccccc} \hline
        & {\bf Tsai89}  & {\bf Andreff99} & {\bf Daniilidis99}  & {\bf Malti10}   & {\bf Heller14}  &  {\bf Alg. \ref{alg_1} }   \\ \hline
$e_X$   & 11.8042  & 57.2739  & 0.0042  & 44.1233  & 0.0042  & 0.0040     \\
Time(s) & 0.0566   & 0.0212   & 0.2014  & 3.8501   & 1.3656  & 1.1441     \\
\hline
\end{tabular}
\end{table}

\begin{table}[!th]
\caption{Numerical results for $AX=ZB$ with parallel rotation axis} \vspace{2mm}
\label{tb_4}
\centering
\begin{tabular}{c|cccccc} \hline
        & {\bf Zhuang94}   & {\bf Dornaika98}  & {\bf Li10}   &  {\bf Heller14} & {\bf Li18}  &  {\bf Alg. \ref{alg_2}}   \\ \hline
$e_X$   & 45.3702  & 112.3677 & 0.0064 &  0.0068 &  21.8262   &  0.0023   \\
$e_Z$   & 259.5928 & 642.9365 & 0.0125 &  0.0382 &  124.8814  &  0.0128   \\
Time(s) & 0.0191   & 0.0179   & 0.0839 & 39.2364 &  2.2217    &  1.5278   \\
\hline
\end{tabular}
\end{table}

\subsection{Measurement estimation with noise}
In practice, the measurement of $B$ is typically estimated using visual processing. Since visual estimation is noisy, this set of experiment aims comparing the robustness of the different methods to disturbances in the measurement of $B$.

The four measurements $\left(A^{(i)}, B^{(i)}\right)_{i=1}^4$ are the same with that in Subsection 5.1. The rotation and translation of $B^{(i)}$ are disturbed by adding zero mean Gaussian noise with increasing standard deviation. Note that the motions $\tilde{B}^{(s)}$ are also disturbed when adding noise to the measurements $B^{(i)}$. The standard deviation of the additive noise increases from 0 to 0.02 in steps of 0.002. For each standard deviation, the average errors of $e_X$ and $e_Z$ are recorded after 10 runs of each method. The robustness testing for $AX=XB$ and $AX=ZB$ with noisy measurements of $B$ are plotted in Figures \ref{fig1} and \ref{fig2}, respectively.

\begin{figure}[!t]
\centering
\includegraphics[width=0.55\textwidth]{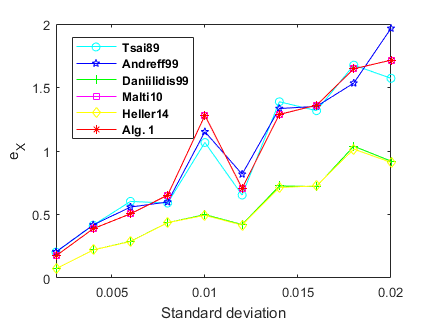}
\caption{Robustness testing for $AX=XB$.}
\label{fig1}
\end{figure}

\begin{figure}[!t]
\centering
\includegraphics[width=\textwidth]{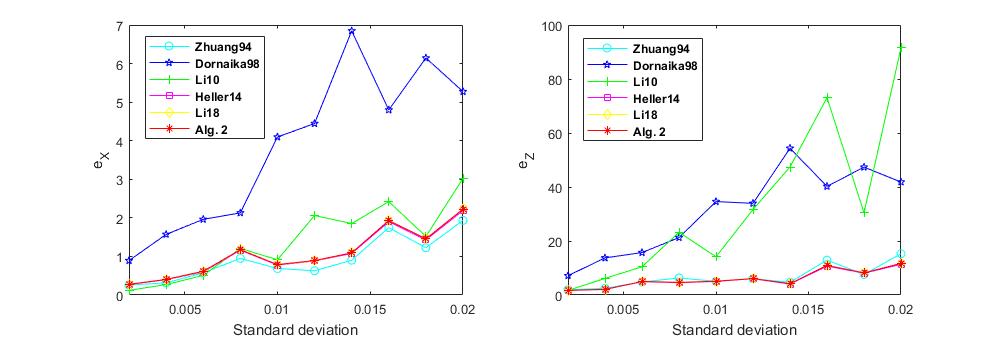}
\caption{Robustness testing for $AX=ZB$.}
\label{fig2}
\end{figure}

\section{Final Remarks}

In this paper, we establish a new dual quaternion optimization method for the hand-eye calibration problem based on the 2-norm of dual quaternion vectors. A two-stage method is also proposed by using the techniques of regularization and patching. However, there are still some problems that need further study.
We have the following final remarks.

1. Can we use some other norms for dual quaternion vectors, e.g. 1-norm, $\infty$-norm, instead of 2-norm in this method?

2. We may also consider some other hand-eye calibration models, such as multi-camera hand-eye calibration.

3. How can we choose the regularization parameter $\gamma$ to improve the efficiency of the method?

4. Can we extend this method to the simultaneous localization and mapping problem?

%We have the following final remarks.

%1. Using other norms for dual quaternion vectors, e.g. 1-norm, $\infty$-norm.

%2. Other hand-eye calibration model.

%3. How to choose the regularization parameter $\gamma$.

\bigskip

{\bf Acknowledgment.}
We are very thankful to Wei Li for helpful discussion and providing the data and Matlab code for the methods in \cite{LLDL18}.
%We would like to thank the referee who carefully read my manuscript and gave very helpful comments.

%{\bf Data availability statement}    The datasets generated during and/or analysed during the current study are available from the corresponding author on reasonable request.form.

% \vspace{100pt}


\begin{thebibliography}{10}
\bibitem{Q22} Qi L, ``Standard dual quaternion optimization and its applications in hand-eye calibration and SLAM", {\sl Communications on Applied Mathematics and Computaion} DOI: 10.1007/s42967-022-00213-1.
\end{thebibliography}
\end{document}